\documentclass[sigconf]{acmart}

\usepackage[utf8]{inputenc}

\usepackage{booktabs} 
\setcopyright{none}

\acmDOI{}

\acmISBN{}

\acmConference[]{}{}{}
\acmYear{}
\copyrightyear{}

\acmPrice{}

\DeclareMathOperator*{\argmax}{arg\,max}

\protect

\usepackage{xspace}

\newcommand{\bigoh}{\ensuremath{\mathcal{O}}}

\newcommand{\squishlist}{
 \begin{list}{$\bullet$}
  { \setlength{\itemsep}{0pt}
     \setlength{\parsep}{3pt}
     \setlength{\topsep}{3pt}
     \setlength{\partopsep}{0pt}
     \setlength{\leftmargin}{1.5em}
     \setlength{\labelwidth}{1em}
     \setlength{\labelsep}{0.5em} } }
\newcommand{\squishend}{
  \end{list}  }

 \newcommand{\attr}[1]{\texttt{\scalebox{.7}[1.0]{#1}}\xspace}

\usepackage{subcaption}

\usepackage{booktabs}

\usepackage{todonotes}
\presetkeys{todonotes}{inline, backgroundcolor=yellow}{}

\usepackage{calc}

\usepackage{units}

\usepackage[linesnumbered,algo2e]{algorithm2e}
\usepackage{tabularx}
\usepackage{colortbl}
\usepackage{dashrule}
\usepackage{array}

\begin{document}
\title{Human-Guided Data Exploration}

\author{Andreas Henelius}
\affiliation{%
\department{Department of Computer Science}
  \institution{Aalto University}
\streetaddress{Otaniemi}
  \city{Espoo} 
  \state{Finland}
}
\email{andreas.henelius@aalto.fi}

\author{Emilia Oikarinen}
\affiliation{%
\department{Department of Computer Science}
  \institution{Aalto University}
\streetaddress{Otaniemi}
  \city{Espoo} 
  \state{Finland}
}
\email{emilila.oikarinen@aalto.fi}

\author{Kai Puolam\"aki}
\affiliation{%
\department{Department of Computer Science}
  \institution{Aalto University}
\streetaddress{Otaniemi}
  \city{Espoo} 
  \state{Finland}
}
\email{kai.puolamaki@aalto.fi}

\renewcommand{\shortauthors}{A. Henelius et al.}

\begin{abstract}
The outcome of the explorative data analysis (EDA) phase is vital for
successful data analysis. EDA is more effective when the user
interacts with the system used to carry out the exploration. In the
recently proposed paradigm of iterative data mining the user controls
the exploration by inputting knowledge in the form of patterns
observed during the process. The system then shows the user views of
the data that are maximally informative given the user's current
knowledge. Although this scheme is good at showing surprising views of
the data to the user, there is a clear shortcoming: the user cannot
steer the process. In many real cases we want to focus on
investigating specific questions concerning the data. This paper
presents the Human Guided Data Exploration framework, generalising
previous research. This framework allows the user to incorporate
existing knowledge into the exploration process, focus on exploring a
subset of the data, and compare different complex hypotheses
concerning relations in the data. The framework utilises a
computationally efficient constrained randomisation scheme. To
showcase the framework, we developed a free open-source tool, using
which the empirical evaluation on real-world datasets was carried
out. Our evaluation shows that the ability to focus on particular
subsets and being able to compare hypotheses are important additions
to the interactive iterative data mining process.
\end{abstract}

%
%



\maketitle

\section{Introduction}
Explorative data analysis (EDA) has a long history and an established
position in the data analysis community, dating back to the 1970s
\cite{tukey:1977}. Tools for visual EDA are designed to present data
in such a way that humans can use their innate pattern recognition
skills when exploring the data. Visual exploration of data becomes
more powerful when the user can interact with the system used to
explore the data; such systems have also been around for a long time
\cite{fisherkeller} and the issue of combining interactive
visualisations and advanced data analysis methods still remains open
\cite{vismasterbook}.

Some of the main goals of data exploration is to allow the user to
gain insight into the data and discover new, interesting aspects of
the data. The limitation of most traditional systems is, however, that
they do not allow the user to incorporate his or her background
knowledge into the data analysis workflow. Recently, a new paradigm
for \emph{iterative data mining} has been proposed
\cite{hanhijarvi:2009,debie:2011a,debie:2011b,debie2013}, in which
the user iteratively explores the data, inputting his or her knowledge
in the form of patterns during exploration. The identified patterns
are taken into account in the further iterative data exploration
steps. The iterative data mining approach has also been successfully
applied in visual EDA \cite{puolamaki:2016,kang:2016b,puolamaki2017}.

In the iterative data mining framework, the knowledge that the user
has about the data is modelled through a probability distribution (the
\emph{background distribution}), over the datasets conforming to the
user's knowledge. As the user learns more about the data, new
knowledge is incorporated as patterns that constrain the distribution
of possible datasets. Initially it is assumed that the user has no
knowledge of the data, in which case the starting distribution is
uninformative, e.g., a maximum entropy distribution or an empirical
distribution obtained by permutations. In this work we use as a
background distribution a uniform distribution over all permutations
over the data set. The user inputs knowledge to the distribution in
the form of constraints \cite{lijffijt2014}, formalised here as tiles.

\begin{figure*}[t]
  \centering
  \begin{tabular}{ccc}
\begin{tabular}{c}\includegraphics[width=0.28\textwidth]{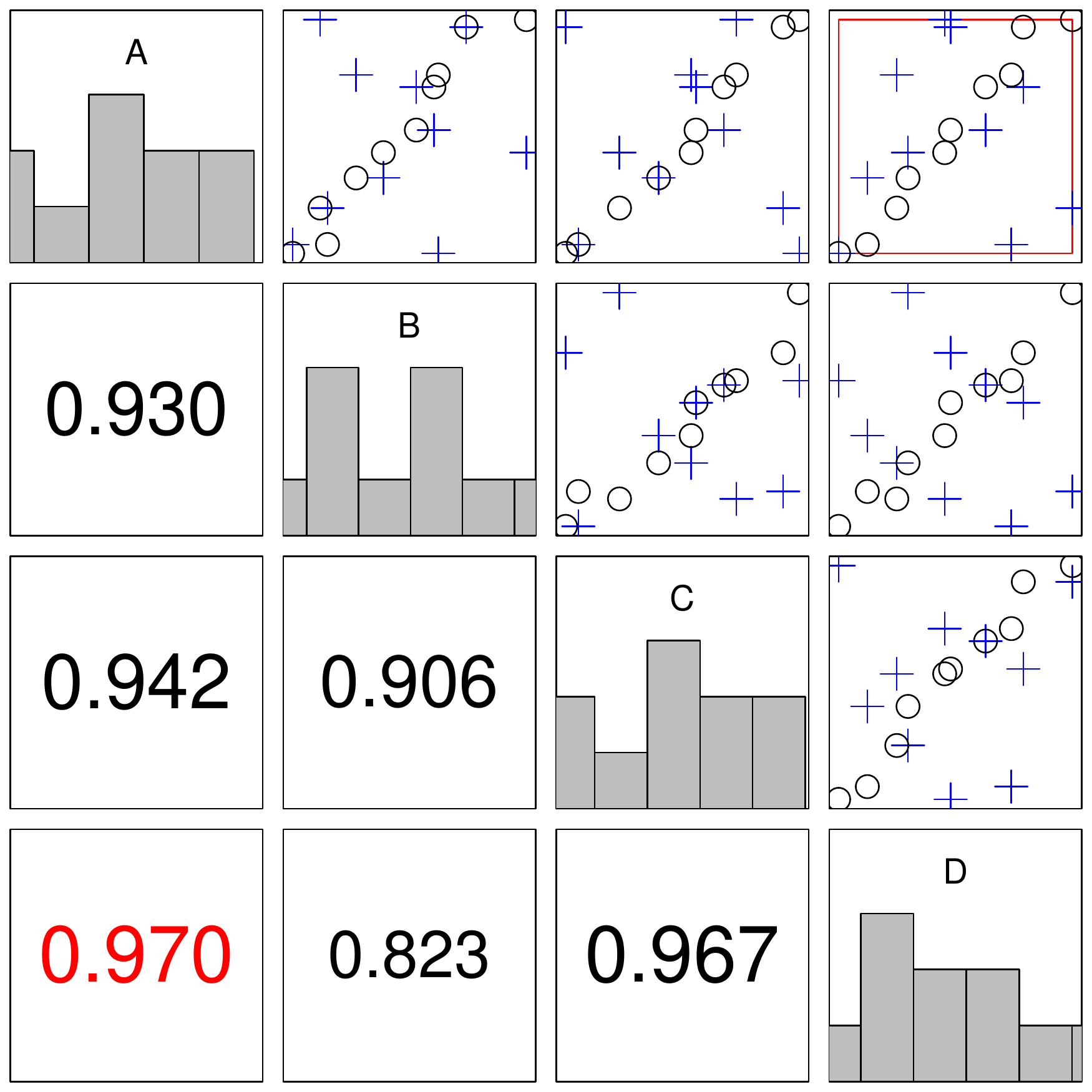}\\(a)\end{tabular}&
\begin{tabular}{c}\includegraphics[width=0.28\textwidth]{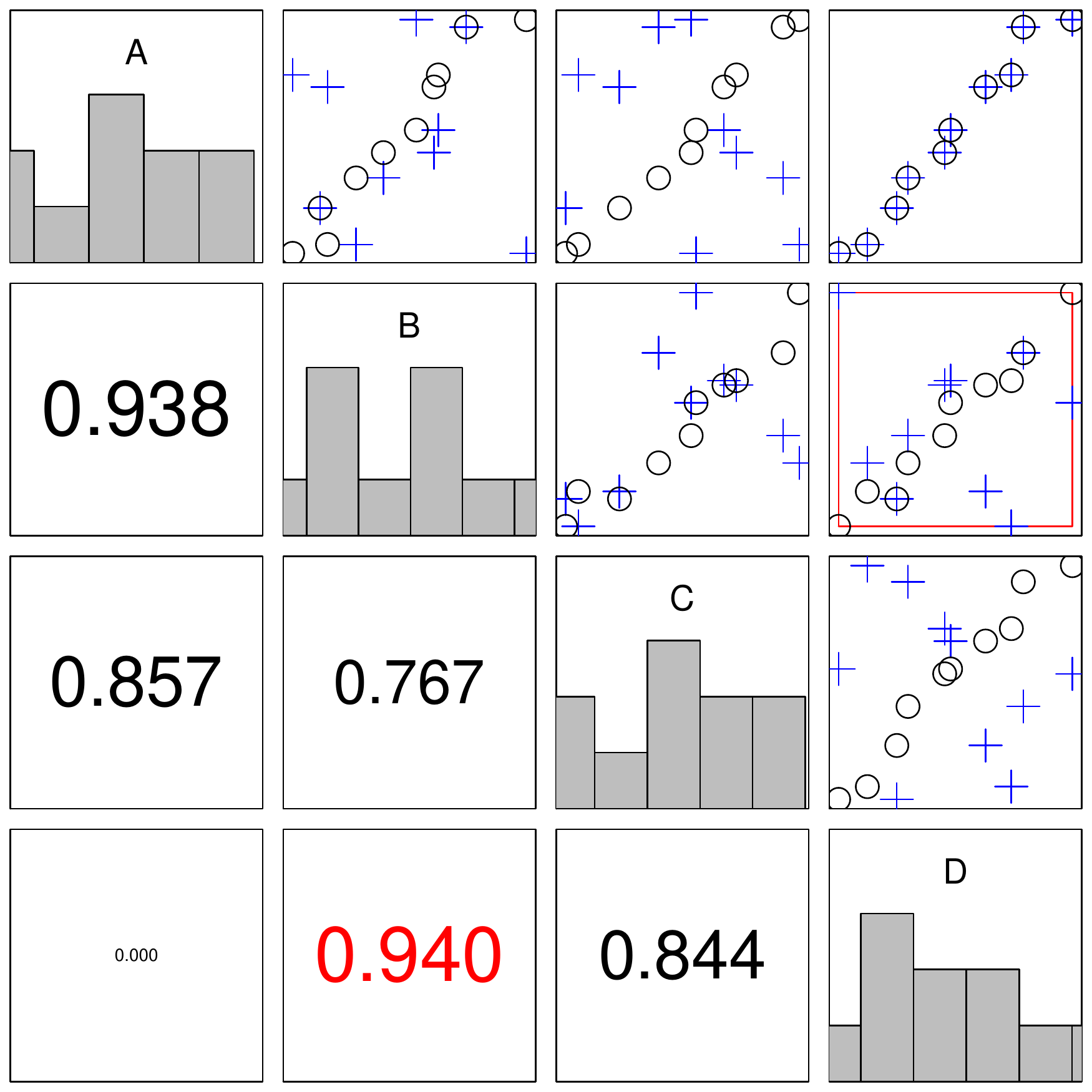}\\(b)\end{tabular}&
\begin{tabular}{c}\includegraphics[width=0.28\textwidth]{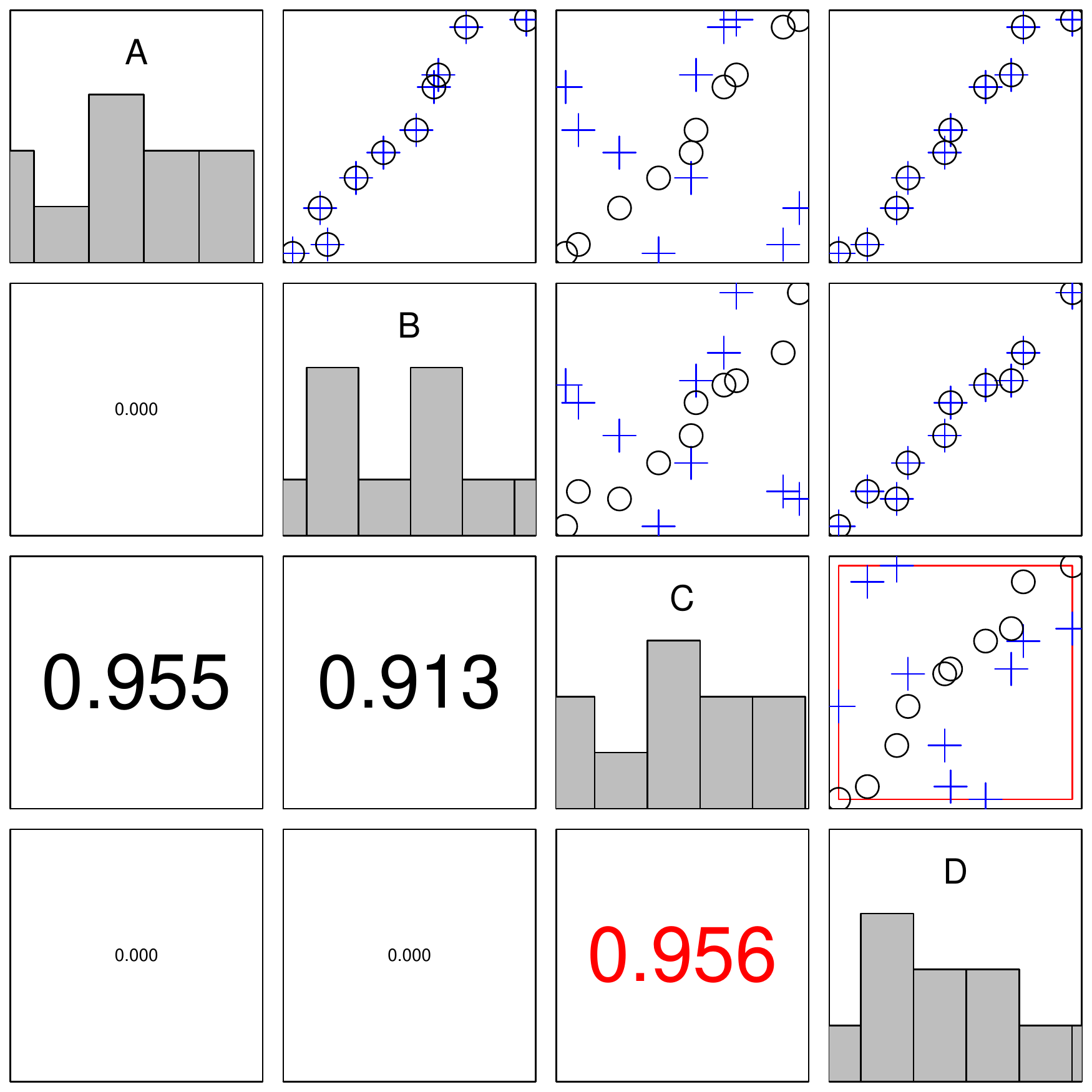}\\(c)\end{tabular}\\\\
\begin{tabular}{c}\includegraphics[width=0.28\textwidth]{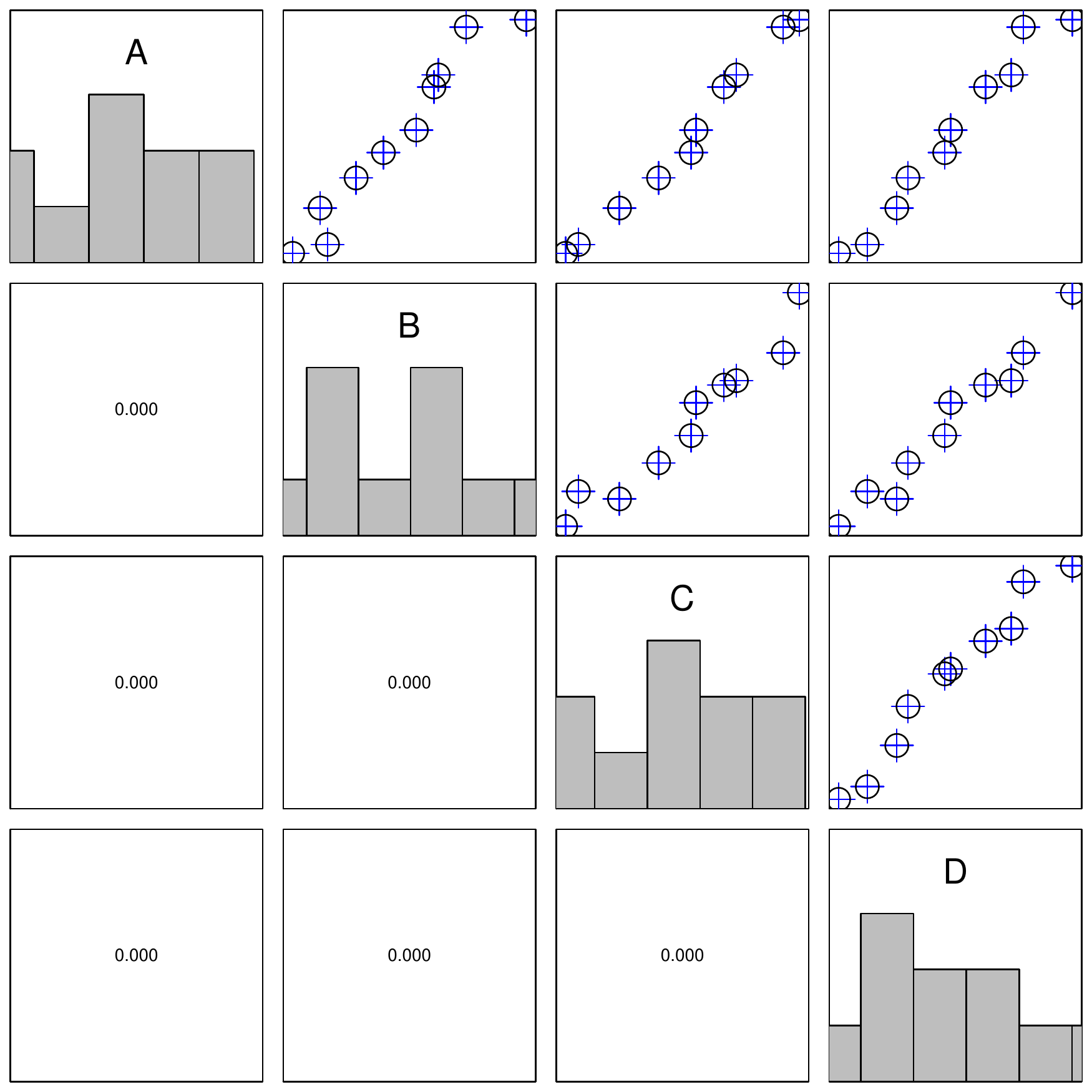}\\(d)\end{tabular}&
\begin{tabular}{c}\includegraphics[width=0.28\textwidth]{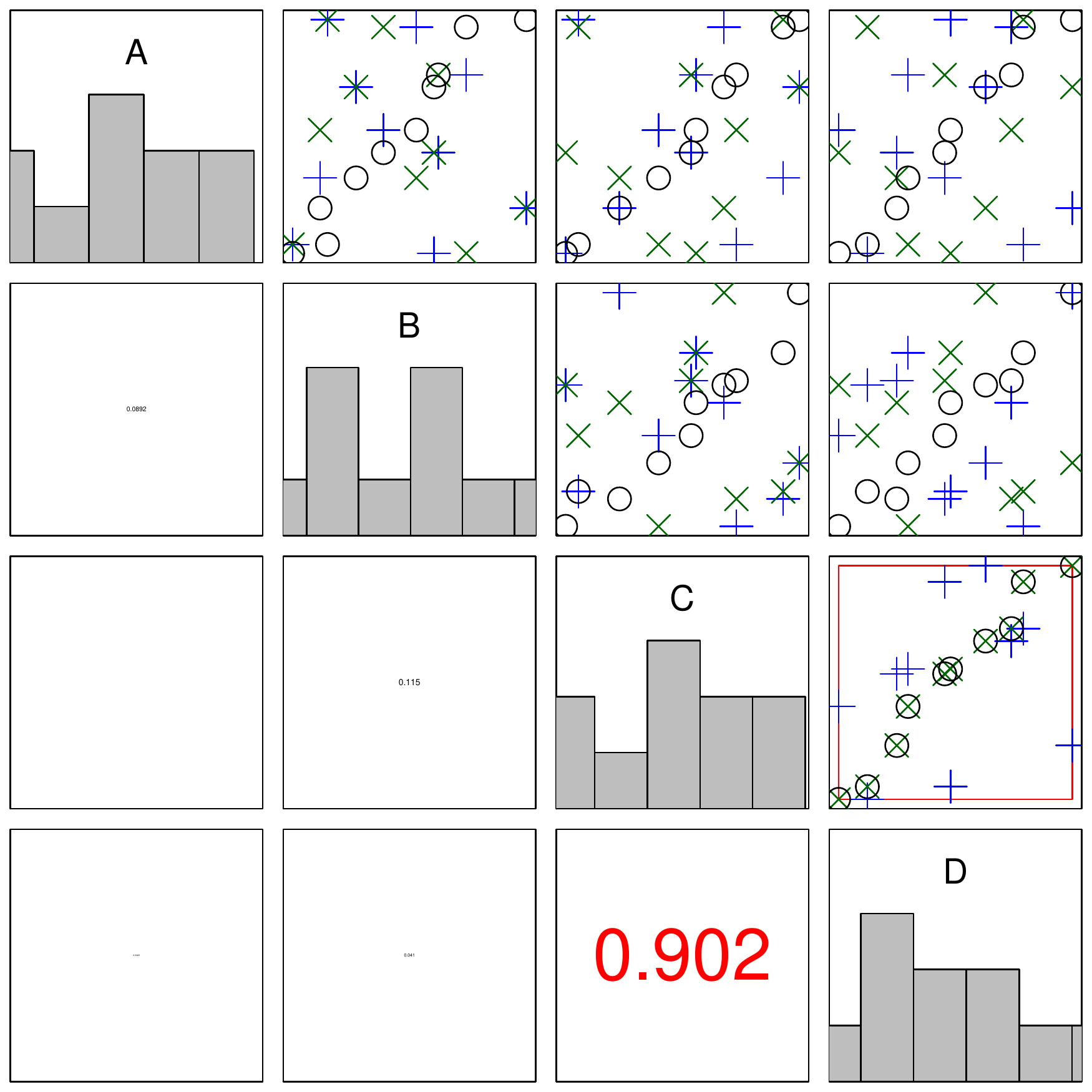}\\(e)\end{tabular}&
\begin{tabular}{c}\includegraphics[width=0.28\textwidth]{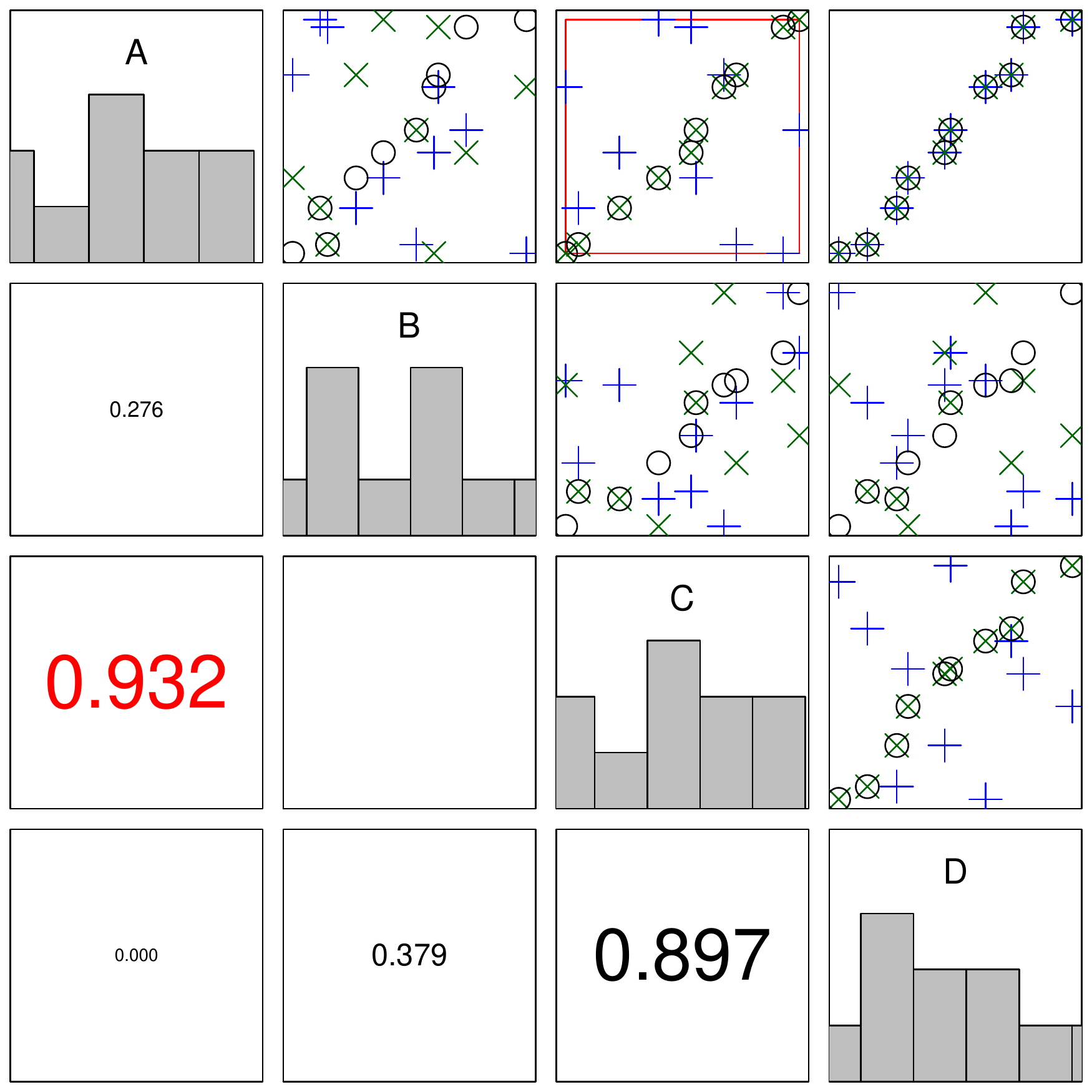}\\(f)\end{tabular}
\end{tabular}
  \caption{Example of visual data exploration with a toy data of 10 items. See the text for discussion.}
    \label{fig:usermodel}
\end{figure*}
\vspace{0.5\baselineskip}\noindent \textbf{Example of iterative data
  mining} Consider the process in Fig.~\ref{fig:usermodel} showcasing
a simple example of exploring a 4-dimensional data set of 10 items
with attributes given by $A$, $B$, $C$, and $D$, respectively. The
data set---shown with black hollow circles---exhibits a simple linear
correlation structure in all dimensions. Initially, however, the user
is unaware of the correlation structure and in his or her
\emph{background model} the attributes are uncorrelated. A sample of
the initial background model obtained by permuting all columns
independently at random, is shown with blue plus-signs in
Fig.~\ref{fig:usermodel}a. The goal of the system is to present the
user with a \emph{maximally informative} view of the data given the
user's current knowledge, i.e., the view in which the data and the
background distribution differ the most. In this example, the possible
views are given by 6 scatterplots over pairs of attributes ($AB$,
$AC$, etc.).

In our example $AD$ is the most informative view. As the measure of
informativeness in this example, we used the difference of the linear
correlation between the actual data and the background distribution:
if the user thinks two attributes are unrelated and observes that they
are in fact correlated, it is both surprising and interesting to the
user. The system presents the view $AD$, leading the user to update
his or her internal model accordingly as shown in
Fig.~\ref{fig:usermodel}b, where columns $A$ and $D$ have been
permuted together with the same permutation. Now, $BD$ is the most
informative view. This is again absorbed by user, leading to
Fig.~\ref{fig:usermodel}c where columns $A$, $B$, and $D$ are permuted
together with the same permutation. After the user has absorbed the
information in view $CD$, he or she has assimilated all relations in
the data after three views and the data and the background model match
perfectly, as shown in Fig.~\ref{fig:usermodel}d.\footnote{The three
  views make sense in this example: by observing the structure in the
  three views $AD$, $BD$, and $CD$, the user can conclude that there
  must be correlation also in the unseen views $AB$, $AC$, and $BC$.}
Note that permuting the data (in contrast to other methods of
modelling the data) preserves the marginal distributions in the
columns, as evidenced by the histograms on the diagonals in
Fig.~\ref{fig:usermodel}.

\vspace{0.5\baselineskip}\noindent \textbf{Human-guided data exploration}
A clear shortcoming in the above described iterative data mining
process is the fact that it \emph{cannot be guided by the user}, since
the system always tries to show the user the view in which the
background distribution differs the most from the actual data. In
other words, the user can only give feedback in the form of known or
observed patterns, but the user cannot direct the exploration process
to answer specific questions and the results are unpredictable,
because---by definition!---the user is a priori unaware of which parts
of the data that differ the most from his or her assumptions.

The fundamental question is then: \emph{how can the user focus the
  search on a specific hypothesis of his or her choice?} Our insight
is, instead of comparing the background distribution and the observed
data, to compare two different hypothesis distributions both formed
using the same background distribution, find the views where the
hypothesis distributions differ most, and show how the data behaves in
these views. In this way the user can establish which of the
hypothesis is broken and by which data patterns (which the user can
then add to the background knowledge).  For example, the user might
want to study whether two groups of attributes are independent. To
this purpose the system would find views that most prominently show
the differences between the two hypotheses, taking the already
observed patterns into account (i.e., the user would not need to see
the same information twice).

Consider again the toy example in Fig.~\ref{fig:usermodel}. Let us
assume that the user is interested only in the nature of the relation
between attributes $C$ and $D$. A natural choice is to take the
current background distribution and formulate two new distributions:
one in which the relation between $C$ and $D$ is fully broken and one
in which it is maintained, and then to find out which views are most
informative concerning this relation.

If the user has not yet observed the data, the user's background model
is as shown in Fig. \ref{fig:usermodel}a. We can form
\textsc{Hypothesis 2} in which the relation between $C$ and $D$ is
broken; in this case this implies the initial background model. In
\textsc{Hypothesis 1} attributes $C$ and $D$ are permuted together and
their relation is preserved. Samples from these two hypotheses are
shown in Fig. \ref{fig:usermodel}e, \textsc{Hypothesis 2} with blue
plus-signs 
and \textsc{Hypothesis 1} with green crosses. Not
surprisingly, the hypotheses differ most in the view $CD$ which the
system then shows the user. Looking at the view $CD$ the user absorbs
the fact that \textsc{Hypothesis 2} is visually quite different from
the data and hence he or she absorbs the relation of $C$ and $D$.%
\footnote{In this work we always compare distributions over data sets
  by taking a sample from the distribution and then comparing these
  samples. Therefore, we only need to be able to compare data sets,
  substantially simplifying our task. In theory, we could also
  directly compare distributions, in our case parametrised by tiles,
  but this problem would be more difficult and specific to the
  parametrisation of the distribution.}

Assume now that the user has observed the view $AD$ and the background
model is thus as shown in Fig. \ref{fig:usermodel}b.
\textsc{Hypothesis 2} corresponds to a model where attributes $A$ and
$D$ are permuted together and $B$ and $C$ independently, and
\textsc{hypothesis 1} to a model where $A$, $C$, and $D$ are permuted
together (if $A$ and $D$ are permuted together, and $C$ and $D$ are
permuted together, then $A$, $C$, and $D$ must all be permuted
together) and $B$ independently. Fig. \ref{fig:usermodel}f shows these
two hypotheses. Now, the view $AC$ gives slightly more information
than $CD$ about the relation of attributes $C$ and $D$! The
explanation for this counterintuitive behaviour is that because the
user knows that $A$ and $D$ are correlated, then the observation of
correlation between $A$ and $C$ would reveal that $C$ and $D$ must be
correlated. Therefore the system could choose to show the user the
view $AC$ as well.

As this simple example shows, the constrained randomisation framework
can be used to guide data exploration. In the simplest case
(Fig. \ref{fig:usermodel}e) this guidance leads to predictable results
(to learn about correlation between two attributes it is useful to
look at the scatterplot of these attributes). However, if we take the
user's background knowledge into account we already obtain non-trivial
insights, as shown by the informativeness of the view $AC$ in
Fig. \ref{fig:usermodel}f. The situation becomes even more non-trivial
when the background model is more complex (here the constraint is
applied to whole views, but in practice the user might not absorb the
view as a whole but only parts of it, e.g., cluster structures or
outliers), instead of looking at relations of single variables we look
at relations between groups of variables (a single view cannot show
all relations between groups of many variables), and when the user is
focusing only on a subset of data items. As an added bonus, we show
later that the above described earlier iterative process in
Figs.~\ref{fig:usermodel}a--d is in fact a special case of our general
framework.

Our framework has three important properties. (i) It is possible to
incorporate existing and new information into the user's background
model, guiding the exploration, i.e., we maintain the property that
the system tries to provide the user with maximally informative views
of the data. (ii) It is possible to focus on exploring aspects in a
subset of the data items and attributes. (iii) The framework allows
the user to \emph{compare hypotheses} concerning the data, so that
specific questions can be answered.  Notice that while we concentrate
on two-dimensional scatterplots as the views to the data in this
paper, the framework is totally general. The data types and the views
can be anything.

\begin{figure}[t]
  \centering
  \includegraphics[width=.37\textwidth]{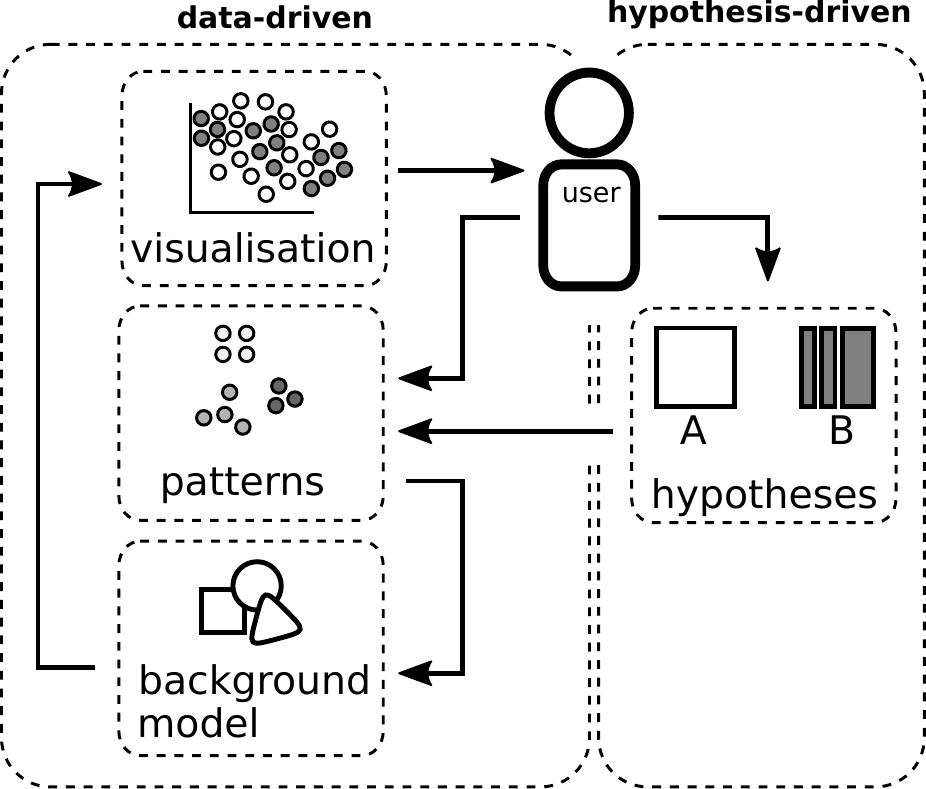}
  \caption{Human-Guided Data Exploration.}
    \label{fig:visualeda}
\end{figure}

The basic high-level workflow in the EDA framework is shown in
Fig.~\ref{fig:visualeda}. The above described iterative exploration
paradigm (\emph{explore}) is described on the left side in the figure
and is termed data-driven exploration. By augmenting the data-driven
exploration by the hypothesis-driven exploration (shown on the right
in the figure), we get the focus-oriented framework presented in this
paper.

\vspace{0.5\baselineskip}\noindent \textbf{Related work} This work is
motivated by \cite{puolamaki:2016} in which a similar system was
constructed using constrained randomisation but without user guidance,
see \cite{DLS:16,kang2016,puolamaki2017} for approaches using the
Maximum Entropy distribution. The Maximum Entropy method has been
proposed \cite{debie:2011a,debie2013} for modelling the user's
knowledge by a background distribution and it has been studied in the
context of dimensionality reduction and EDA
\cite{DLS:16,kang2016,puolamaki2017}.  To the best of our knowledge,
the framework presented in this paper is the first instance in which
this background distribution can be updated both with direct
interaction \emph{and} user guidance, thus providing a principled
method of EDA.

 Many other special-purpose methods have been developed for active
 learning in diverse settings, e.g., in classification and ranking, as
 well as explicit models for user preferences. As these approaches are
 not targeted at data exploration, we do not review them
 here. Several special-purpose methods have been developed
 for visual iterative data exploration in specific contexts, e.g., for
 item-set mining and subgroup discovery
 \cite{boley2013,dzyuba2013,vanleeuwen2015,paurat2014}, information
 retrieval \cite{ruotsalo2015}, and network analysis \cite{chau2011},
 also see \cite{vismasterbook} for a summary. 
 
 Finally, the system presented here can also be considered an instance
 of \emph{visually controllable data mining} \cite{puolamaki2010},
 where the objective is to implement advanced data analysis methods
 understandable and efficiently controllable by the user. Our approach
 satisfies the properties of a visually controllable data mining
 method \cite[Sec. II B]{puolamaki2010}: (VC1) the data and the model
 space are presented visually, (VC2) there are intuitive visual
 interactions allowing the user to modify the model space, and (VC3)
 the method is fast enough for visual interaction.

\vspace{0.5\baselineskip}\noindent \textbf{Contributions} We make the
following contributions. (i) We present a novel framework for focused
interactive iterative exploratory data analysis. The framework extends
previous research by generalising the paradigm of iterative data
mining to be user-focusable in terms of specific hypotheses specified
by the user. (ii) We present an implementation of a free open source
software tool (MIT license) realising the framework presented in this
paper. Using the tool it is possible to interactively explore data,
and all experiments in this paper are carried out using this
tool. (iii) We demonstrate the three use-cases of exploring, focusing
and comparing (with an emphasis on the latter two) using different
datasets.

\vspace{0.5\baselineskip}\noindent \textbf{Outline} In
Sec. \ref{sec:methods}, we introduce the methods, and in
Sec. \ref{sec:framework}, we present our novel framework for
exploratory data analysis.  In Sec. \ref{sec:experiments}, we present
our experimental evaluation together with an overview of our tool. We
present both investigations carried out on real datasets and an
evaluation of the scalability of our framework for EDA. We conclude
with a discussion in Sec. \ref{sec:conclusions}.

\section{Methods}\label{sec:methods}

We introduce the theory and methods related to the constrained
randomisation of data matrices using tiles in this section.

Let $X$ be an $n \times m$ data matrix (data set) sampled from a
distribution $P$. Here $X(i,j)$ denotes the $i$th element in column
$j$. Each column $X(\cdot, j),\ j \in [m]$, is an \emph{attribute} in
the dataset, where we used the shorthand $[m] = \{1, \ldots, m\}$. Let
$D$ be a finite set of domains (e.g., quantitative or categorical) and
let $D(j)$ denote the domain of $X(\cdot, j)$. Also let $X(i, j) \in
D(j) \textrm{ for all } i \in [n],\ j \in [m]$, i.e., all elements in
a column belong to the same domain, but different columns can have
different domains.

We define a \emph{permutation} $\widehat X$ of the data matrix $X$ as follows.
\begin{definition}[Permutation]\label{def:permutation}
  Let ${\mathcal P}$ denote the set of permutation functions of length
  $n$ such that $\pi:[n]\mapsto[n]$ is a bijection for all
  $\pi\in{\mathcal P}$, and denote by
  $\left(\pi_1,\ldots,\pi_m\right)\in{\mathcal P}^m$ the vector of
  column-specific permutations. A permutation of the data matrix $X$
  is then given as $\widehat X(i,j)=X(\pi_j(i),j)$.
\end{definition}
We express the relations in the data matrix $X$ by \emph{tiles}.
\begin{definition}[Tile]
A tile is a tuple $t = (R, C)$, where $R \subseteq [n]$ and $C
\subseteq [m]$.
\end{definition}
The tiles considered here are combinatorial (in contrast to
geometric), meaning that rows and columns in the tile do not need to
be consecutive. In our illustrations in this paper we, however, 
use geometric tiles for clarity of presentation. In an unconstrained
case, there are $(n!)^m$ allowed vectors of permutations. The tiles
constrain the set of allowed permutations as follows.
\begin{definition}[Tile constraint]
  \label{def:tileconstraint}
  Given a tile $t = (R,C)$, the vector of permutations
  $\left(\pi_1,\ldots,\pi_m\right)\in{\mathcal P}^m$ is allowed by $t$
  iff the following condition is true for all $i\in[n]$, $j\in[m]$,
  and $j'\in[m]$:
\begin{equation}
i\in R\wedge j\in C\wedge j'\in C\implies
\pi_j(i)\in R\wedge \pi_j(i)=\pi_{j'}(i).
\end{equation}
Given a set of tiles $T$, a set of
permutations is allowed iff it is allowed by all $t\in T$.
\end{definition}
A tile defines a subset of rows and columns, and the rows in this
subset are permuted by the same permutation function in each column in
the tile. In other words, the \emph{relations between the columns
  inside the tile are preserved}.  Notice
that the identity permutation is always an allowed permutation.

\begin{figure}[t]
  \centering
  \includegraphics[width=.5\textwidth]{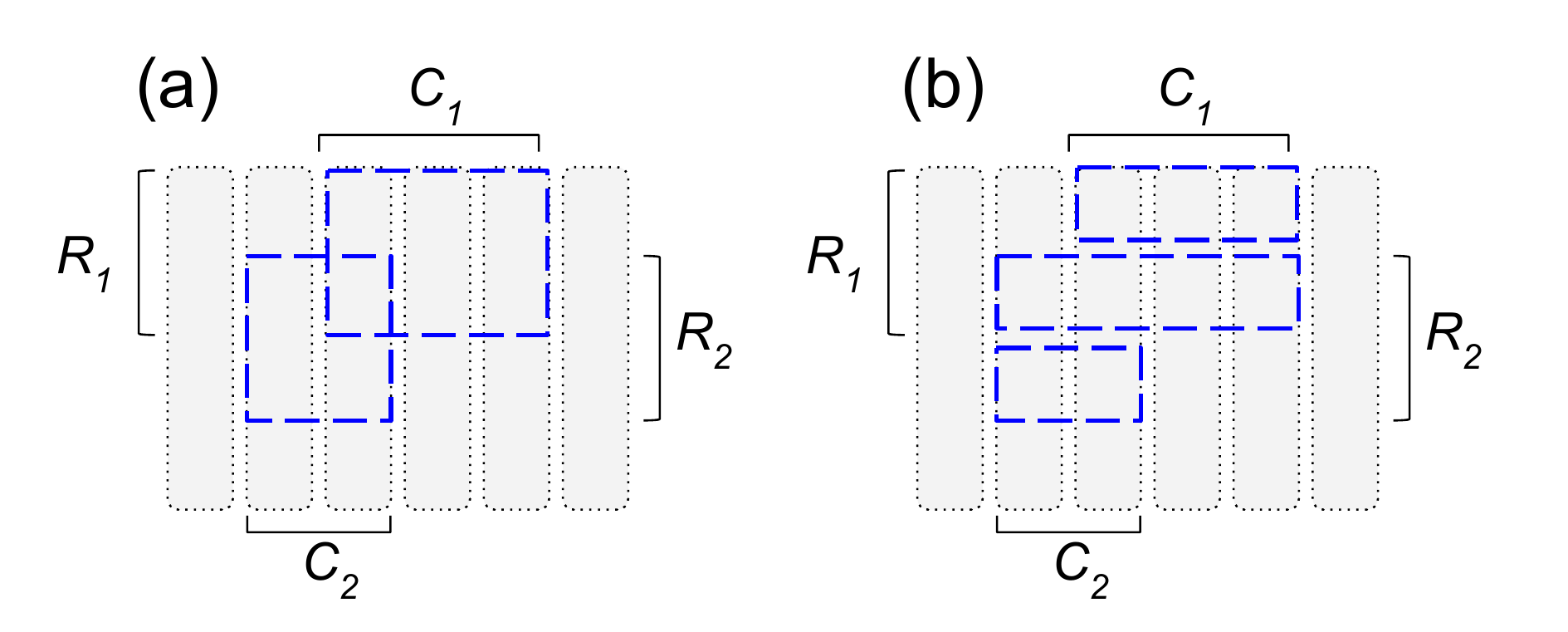}
  \caption{(a) A set of
    tiles, and  (b) an  equivalent tiling.     \label{fig:mergingtiles}}
\end{figure}

Our main problem is to draw samples from the uniform distribution over
allowed permutations, defined as follows.
\begin{definition}[Sampling problem]
\label{def:samplingproblem}
Given a set of tiles $T$, the sampling problem is to draw samples
uniformly at random from vectors of permutations in ${\mathcal P}^m$
such that the vectors of permutations are allowed by $T$.
\end{definition}
The sampling problem is trivial when the tiles are
non-overlapping. However, in the case of overlapping tiles, multiple
constraints can affect the permutation of the same subset of rows and
columns and this issue must be resolved. To this end, we need to
define the equivalence of two sets of tiles.
\begin{definition}[Equivalence of sets of tiles]
\label{def:equivalence}
  Let $T$ and $T'$ be two sets of tiles. Let $P\subseteq
  \mathcal{P}^m$ be the vector of permutations allowed by 
  $T$ and $P'\subseteq \mathcal{P}^m$ be the vector of permutations
  allowed by  $T'$, respectively. We say that $T$ is
  \emph{equivalent} to $T'$ iff $P=P'$.
\end{definition}
Equivalence of sets of tiles hence means that the same constraints are
enforced on the permutations.

We say that a set of tiles $\mathcal{T}$ such that no tiles
overlap, is a \emph{tiling} (we denote tilings using calligraphic
letters).
Next, we show that there always exists a tiling equivalent to a set of
tiles.
\begin{theorem}[Merging of tiles]
Given a set of (possibly overlapping) tiles $T$, there exists a tiling $\mathcal{T}$ that 
 is equivalent to $T$.
\end{theorem}
\begin{proof}
Let $t_1 = (R_1, C_1)$ and $t_2 = (R_2, C_2)$ be two overlapping
tiles, as in Fig.~\ref{fig:mergingtiles}(a). Each tile describes a set of
constraints on the allowed permutations of the rows in their
respective column sets $C_1$ and $C_2$. A  tiling $\{t'_1, t'_2, t'_3\}$  equivalent to $\{t_1,t_2\}$ is given by:
\begin{displaymath}
  t'_1 = (R_1\setminus R_2, C_1), \qquad
  t'_2 = (R_1\cap R_2, C_1\cup C_2), \qquad
  t'_3 = (R_2\setminus R_1, C_2),
\end{displaymath}
as shown in Fig.~\ref{fig:mergingtiles}(b). Tiles $t'_1$ and $t'_3$
represent the non-overlapping parts of $t_1$ and $t_2$ and the
permutation constraints by these parts can be directly met. Tile
$t'_2$ takes into account the combined effect of $t_1$ and $t_2$ on
their intersecting row set, in which case the same permutation
constraints must apply to the union of their column sets. It follows
that these three tiles are non-overlapping and enforce the combined
constraints of tiles $t_1$ and $t_2$. Hence, a tiling can be
constructed by iteratively resolving overlap in a set of tiles until
no tiles overlap.
\end{proof}
Notice that merging overlapping tiles leads to tiles with
larger column sets (wider tiles) but smaller row sets (lower
tiles). The limiting case is a fully-constrained situation where each
row is a separate tile and only the identify permutation is allowed.

\subsection{Algorithm for merging tiles}
Merging a new tile into a tiling where, by definition, all tiles are
non-overlapping, can be done efficiently using the Alg.
\ref{alg:merge}, since we only need to consider the overlap between
the new tile and the existing tiles in the tiling. This is similar to
merging statements in \cite{kalofolias:2016}. Our algorithm has two
steps and works as follows.

We first initialise a hash map (line 1). We then iterate over each row
in $R$  for tile $t=(R,C)$ (lines 2--11). Each such row represents the IDs
of the tiles with which the new tile overlaps at that row. For each
row we check if this row has been seen before, i.e., if the row is a
key in $S$ (line 4). If this is the first time this row is seen, we
initialise $S(K)$ to be a tuple (line 5). We refer to the elements in
the tuple by name as $S(K)_\mathrm{rows}$, and on line 6 we store the
current row index in the tuple. On line 7 we store the unique tile IDs
in the tuple $S(K)_\mathrm{id}$. If the row has been seen before, we
update the row set associated with these tile IDs (line 9). After this
step, $S$ contains tuples of the form (\emph{rows}, \emph{id}):
\emph{id} specifies the IDs of the tiles with which $t$ overlaps at
the rows specified by \emph{rows}.

We then proceed to update the tiling (lines 12--16). We first
determine the currently largest tile ID in use (line 12), after which
we iterate over the tuples in $S$. For each of the tuples we must
update the tiles having IDs of $S(K)_\mathrm{id}$, and we therefore
find the columns associated with these tiles on line 14. After this,
we update the ID of the affected overlapping tiles on line 15, and
increment the tile ID counter. Finally, we return the updated tiling
on line 17. The time complexity of this algorithm is $\bigoh{(n m)}$.

\begin{algorithm2e}[t!]
  \SetKwInOut{Input}{input}
  \SetKwInOut{Output}{output}

  \Input{(1) A tiling $\mathcal{T}$ ($n \times m$
    data matrix where each element is the ID of the tile
    it belongs to) \\(2) A new tile $t = (R,
    C)$.}
  \Output{The tiling $\mathcal{T}$, where  $t$ has been merged with $\mathcal{T}$.}

  $S \leftarrow \texttt{HashMap}$\;

  \For{$i \in R$} {   
    $K \leftarrow \mathcal{T}(i, C)$\;
    \uIf { $K \notin \texttt{keys}(S) $}{
      $S(K) \leftarrow \texttt{Tuple}$\;
      $S(K)_\mathrm{rows} \leftarrow \{i\}$\;
      $S(K)_\textrm{id} \leftarrow \texttt{unique}( \mathcal{T}(i, C))$\;
    }      \Else{
      $S(K)_\textrm{rows} \leftarrow  S(K)_\textrm{rows} \cup \{i\}$\;
    }
  }

  $p_\mathrm{max} \leftarrow \max(\mathcal{T}(R,C))$\;
  \For{$K \in \texttt{keys}(S)$}{
    $C' = \left\{ c \mid \mathcal{T}(S(K)_\mathrm{rows}, c) \in S(K)_\textrm{id}  \right\}$ \;
    $\mathcal{T}\left(S(K)_\textrm{rows}, C'\right) \leftarrow p_\mathrm{max} + 1$; \qquad  $p_\mathrm{max} \leftarrow p_\mathrm{max} + 1$\;
      }
  \Return{$\mathcal{T}$}

  \caption{\label{alg:merge} Merging tiles. The function
    \texttt{HashMap} denotes a hash map. The value in a hash map $H$
    associated with a key $x$ is $H(x)$ and $\texttt{keys}(H)$ gives
    the keys of $H$. The function \texttt{Tuple} creates a (named)
    tuple. An element $a$ in a tuple $w = (\mathrm{a}, \mathrm{b})$ is
    accessed as $w_\mathrm{a}$. The function \texttt{unique} returns
    the unique elements of an array.}
\end{algorithm2e}

\section{Exploration Framework}\label{sec:framework}
The task of the user in the exploratory framework shown in
Fig.~\ref{fig:visualeda} is to compare two data samples, corresponding
to different \emph{hypotheses}, and draw conclusions based on this. In
this section we present how the background model and these 
hypotheses are formed and their
relation to typical exploratory data mining tasks.

\textbf{Background model.} The user's knowledge concerning relations
in the data are described by tiles (Def.~\ref{def:tileconstraint}). As
the user views the data the user can highlight relations he or she has
absorbed by tiles. For example, the user can mark an observed cluster
structure with a tile involving the data points in the cluster and the
relevant dimensions. We denote the set of user-defined tiles by
${\mathcal T}_u$. A sample from the background model is defined by
first sampling a permutation vector uniformly from the set of
permutation vectors allowed by the tiling ${\mathcal T}_u$
(Def. \ref{def:samplingproblem}) and then obtaining the permuted data
set (Def. \ref{def:permutation}).

\textbf{Hypothesis.} The hypothesis
tilings specify the hypotheses that the user is investigating and
provide a flexible method for investigating different questions
concerning relations in the data. 
\begin{definition}[Hypothesis tilings]\label{def:hypothesis} 
Given a subset of rows $R\subseteq[n]$, a subset of columns
$C\subseteq[m]$, and a $k$-partition of the columns given by $C_1,\ldots,C_k$,
such that $C=\cup_{i=1}^k{C_k}$ and $C_i\cap C_j=\emptyset$ if $i\ne
j$, a \emph{hypothesis tiling} is given by 
$\mathcal{T}_{H_1}=\{(R,C)\}$ and
$\mathcal{T}_{H_2}=\cup_{i=1}^k{\{(R,C_i)\}}$, respectively.
\end{definition}
Because a hypothesis tiling defines a question about the data it can
be specified before observing any data points.

The hypothesis tiling defines the items and attributes of interest and
the relations between the attributes that the user is interested in.
In our general framework, the user compares samples from two
distributions, each sample corresponding to a hypothesis. These two
samples are obtained by sampling permutations consistent with
$\mathcal{T}_1=\mathcal{T}_u + \mathcal{T}_{H_1}$ and
$\mathcal{T}_2=\mathcal{T}_u + \mathcal{T}_{H_2}$, respectively. We
here use '$+$' with a slight abuse of notation to denote the operation
of merging tilings into an equivalent tiling (Alg.~\ref{alg:merge}).
\textsc{Hypothesis 1} ($\mathcal{T}_1$) corresponds to a hypothesis
where all relations in $(R,C)$ are preserved, and \textsc{hypothesis
  2} ($\mathcal{T}_2$) to a hypothesis where there are no unknown (in
addition to those relations in the background model) relations between
attributes in the partitions $C_1,\ldots,C_k$ of $C$. Views that show
differences between these two hypotheses are most informative in
describing how the data set differs with respect to these two
hypotheses.

For example, if the columns are partitioned into two groups $C_1$ and
$C_2$ the user is interested in relations between attributes in $C_1$
and $C_2$, but not in relations within $C_1$ or $C_2$. On the other
hand, if the partition is full, i.e., $k=|C|$ and $|C_i|=1$ for all
$i\in [k]$, then the user is interested in \emph{all} relations
between the attributes. In the latter case, the special case of
$R=[n]$ and $C=[m]$ indeed reduces to the unguided iterative data
mining in Figs. \ref{fig:usermodel}a--d. We define the following
\emph{special cases} of hypotheses.

\emph{Exploration} is the data-driven scenario on the left in
Fig.~\ref{fig:visualeda}, where the user explores the data in order to
discover interesting relations, without directing the search in any
way. This is accomplished by comparing the original data with a sample
corresponding to the user tiling. This is realised using the following
hypothesis tilings: $\mathcal{T}_{H_1} = \{([n], [m])\}$ and
$\mathcal{T}_{H_2} = \{([n],\{j\}) \mid j\in [m]\}$. See
Fig. \ref{fig:usermodel}a--d for an example.

\emph{Focusing} means that the user wants to explore features in the
data limited to a subset of items and attributes. As described in the
introduction, it is often beneficial to focus on the interesting
aspects in the data instead of removing the data outside one's
focus. Hence, 
this is a convenient way to focus on exploring a subset of
the data. This is realised using the following hypothesis tilings: 
$\mathcal{T}_{H_1}=\{(R,C\}$ and
$\mathcal{T}_{H_2}=\{(R,\{j\})\mid j\in C\}$, where $R$ and $C$ are the sets of rows and columns, respectively, of interest. 
See Fig. \ref{fig:usermodel}e--f for an example.

\textbf{Most informative view.} The only remaining problem to solve is
how to find a view in which the distributions defined by the two
hypotheses $\mathcal{T}_1$ and $\mathcal{T}_2$ differ the most. The
answer to this question depends on the type of data and the
visualisation chosen. For example, the visualisations or measures of
difference would be different for categorical or real valued data. For real
valued data we can use projections that mix attributes (such as
principal components). In this work we consider only views that
involve two attributes and can be shown as scatterplots over these
attributes. The problem then reduces to finding the most informative
of these views specified by a pair of attributes.

Here we use the following \emph{measure of interestingness} to choose
the maximally informative view presented to the user.
\begin{definition}[Maximally informative view]
\label{def:interestingness}
  Given two data samples $X_1$ and $X_2$, the maximally informative view
is given by columns $i$ and $j$ subject to
\begin{displaymath}
\argmax_{i,j \in [m], \; i \neq j} | \mathrm{cor}(X_1(\cdot, i), X_1(\cdot, j))^2 - \mathrm{cor}(X_2(\cdot, i), X_2(\cdot, j))^2 |,
\end{displaymath}
where $\mathrm{cor}$ denotes Pearson's correlation.
\end{definition}
While we have here chosen to use correlation due to ease of
interpretability for the user when considering two-dimensional data
projections, it should be noted that also other measures of distance
between distributions can be used.

\section{Experiments}\label{sec:experiments}
In this section we first describe a software tool showcasing our
method. We then continue with a discussion of the scalability of the
method presented in this paper. Finally, we consider use of the tool
in the exploration of real-world datasets.

\begin{figure}[t]
  \centering
  \includegraphics[width=.5\textwidth]{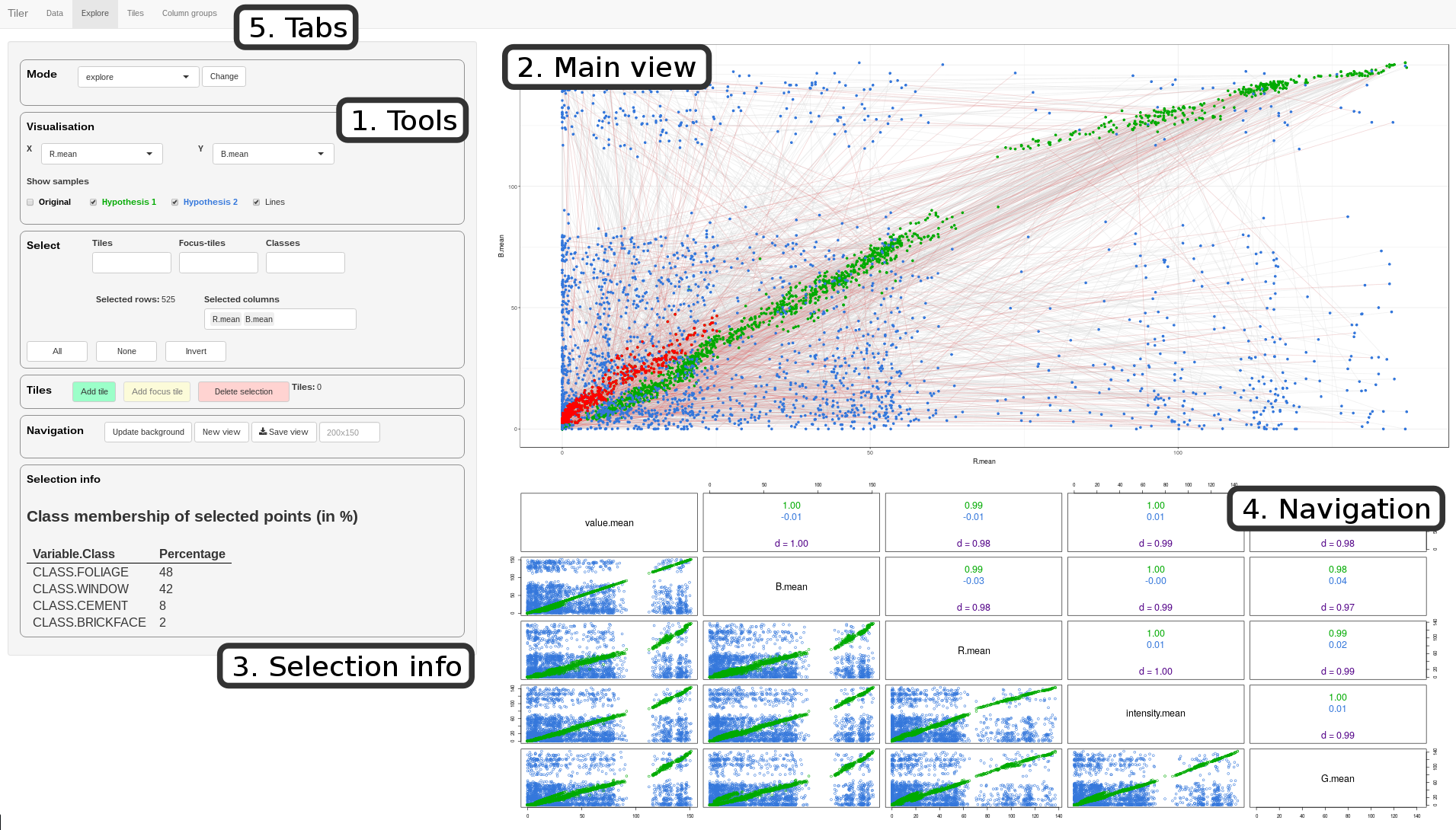}
  \caption{The main user interface of the software tool.}
  \label{fig:ui}
\end{figure}

\begin{table}[ht]
  \centering
\begin{tabular}{rrrrrr}
  \toprule
    $n$ & $m$ & $k$ & $t_\mathrm{init}$ ($s$) & $t_\mathrm{permute}$
                                                ($s$) &
                                                        $t_\mathrm{max, add}$ ($s$)\\
  \midrule
$ 5000$ & $ 50$ & $ 25$ & $0.04$ & $0.01$ & $0.06$ \\
$ 5000$ & $ 50$ & $ 50$ & $0.04$ & $0.02$ & $0.07$ \\
$ 5000$ & $ 50$ & $100$ & $0.03$ & $0.02$ & $0.08$ \\
$ 5000$ & $100$ & $ 25$ & $0.08$ & $0.03$ & $0.09$ \\
$ 5000$ & $100$ & $ 50$ & $0.08$ & $0.04$ & $0.18$ \\
$ 5000$ & $100$ & $100$ & $0.07$ & $0.03$ & $0.12$ \\
$10000$ & $ 50$ & $ 25$ & $0.07$ & $0.04$ & $0.14$ \\
$10000$ & $ 50$ & $ 50$ & $0.08$ & $0.06$ & $0.20$ \\
$10000$ & $ 50$ & $100$ & $0.07$ & $0.06$ & $0.41$ \\
$10000$ & $100$ & $ 25$ & $0.34$ & $0.10$ & $0.21$ \\
$10000$ & $100$ & $ 50$ & $0.15$ & $0.14$ & $0.34$ \\
$10000$ & $100$ & $100$ & $0.15$ & $0.07$ & $0.53$ \\
   \bottomrule
\end{tabular}
  \caption{Scalability of the algorithm. The columns show the number
    of rows ($n$) and columns ($m$) in the tiling and the number of
    user tiles ($k$). The other columns show the time in seconds to
    initialise the tiling, permute the data matrix, and the maximum
    time for adding a tile.}
  \label{tab:runningtime}
\end{table}

\subsection{Software tool}
We developed an interactive tool to demonstrate the use of the above
presented randomisation approach for exploratory data analysis of
generic data sets containing real-valued and categorical
attributes. The tool runs in a web browser and is developed in R
\cite{Rproject} using \textsc{shiny} \cite{shiny}. The tool is
released as free software under the MIT license.\footnote{Available
  from \url{https://github.com/aheneliu/tiler} together with the code
  used for the experiments.} Fig.~\ref{fig:ui} shows the main user
interface of the software tool.

The tool supports the full framework described in
Sec. \ref{sec:framework}, including the special cases
\emph{exploration} (the original iterative data mining framework) and
\emph{focusing}. The tool is interactive and selection of data items
can be made by brushing. Different hypotheses can be formed by
grouping columns. Tiles can both be added and removed and different
data projections can be chosen by the user. The tool models the two
hypotheses being compared by the user and shows the data points
belonging to samples from these models in different colours ({\sc
  Hypothesis 1} in green and {\sc Hypothesis 2} in blue). The tool
also suggests the maximally informative view
(Def.~\ref{def:interestingness}). For simplicity we here chose to use
scatterplots of attribute pairs, but it would be straightforward to
supplement it with visualisations using, e.g., projections as in
\cite{puolamaki:2016,puolamaki2017}. To help the user understand the
data better, class information of selected points (for data having
classes) is shown. Finally, a scatterplot matrix shows a few of the
most interesting variables in the data, which helps the user to
quickly get an overview of the data and to discover patterns

\subsection{Scalability}
We here investigate the performance of the tiling algorithm in terms
of wall-clock running time for interactive use. The experiments were
run on a MacBook Pro laptop with a 2.5 GHz Intel Core i7 processor and
R 3.4.3 \cite{Rproject}.

The results are presented in Tab. \ref{tab:runningtime}. We initialise
a tiling to have $n \in \{5000, 10000\}$ rows and $m \in \{50, 100\}$
columns (this gives $t_\mathrm{init}$). We then add
$k\in\{25,50,100\}$ randomly created tiles (the dimensions of the tile
are 10\% of the row and column space of the data matrix, respectively)
to the tiling and permute the data, giving $t_\mathrm{permute}$. The
maximum time needed to add a tile to a tiling (essentially running
Alg. \ref{alg:merge}) is given by $t_\mathrm{max,add}$.

We note that the initialisation time and time needed to permute the
data is negligible. The time needed to add a tile increases with
the size of the data. The method is clearly fast enough for
interactive use, since even for a dataset with a million elements
addition of a new tile and permutation of the data can be done in less
than a second. It should also be noted that large datasets can
typically be downsampled, if needed, without loss of characteristics
for visual exploration, because usually it is not possible to
visualise millions of data points anyway.

 \begin{figure*}[t]
   \centering
   \begin{tabular}{ccc}
   \includegraphics[width=.28\textwidth]{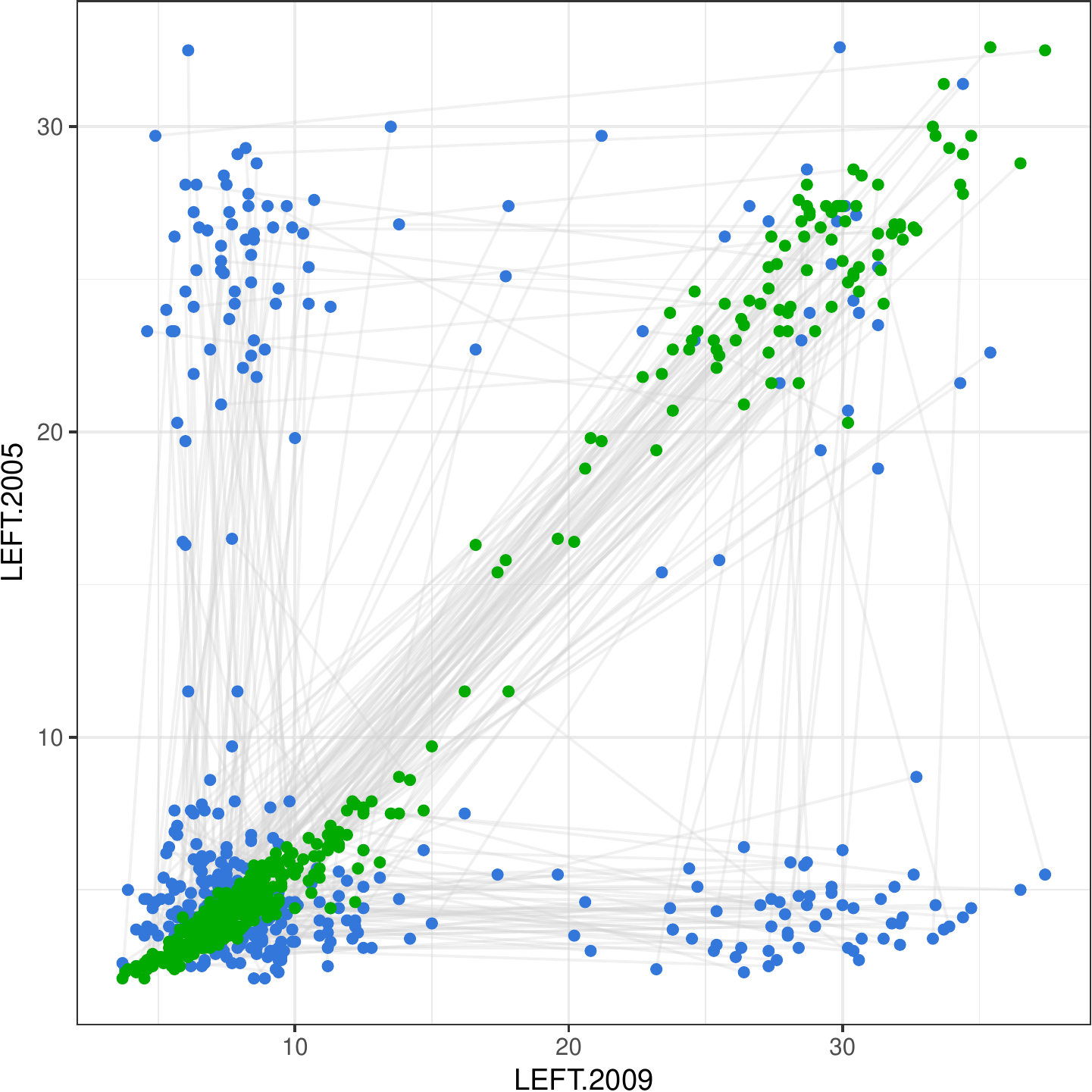} &
   \includegraphics[width=.28\textwidth]{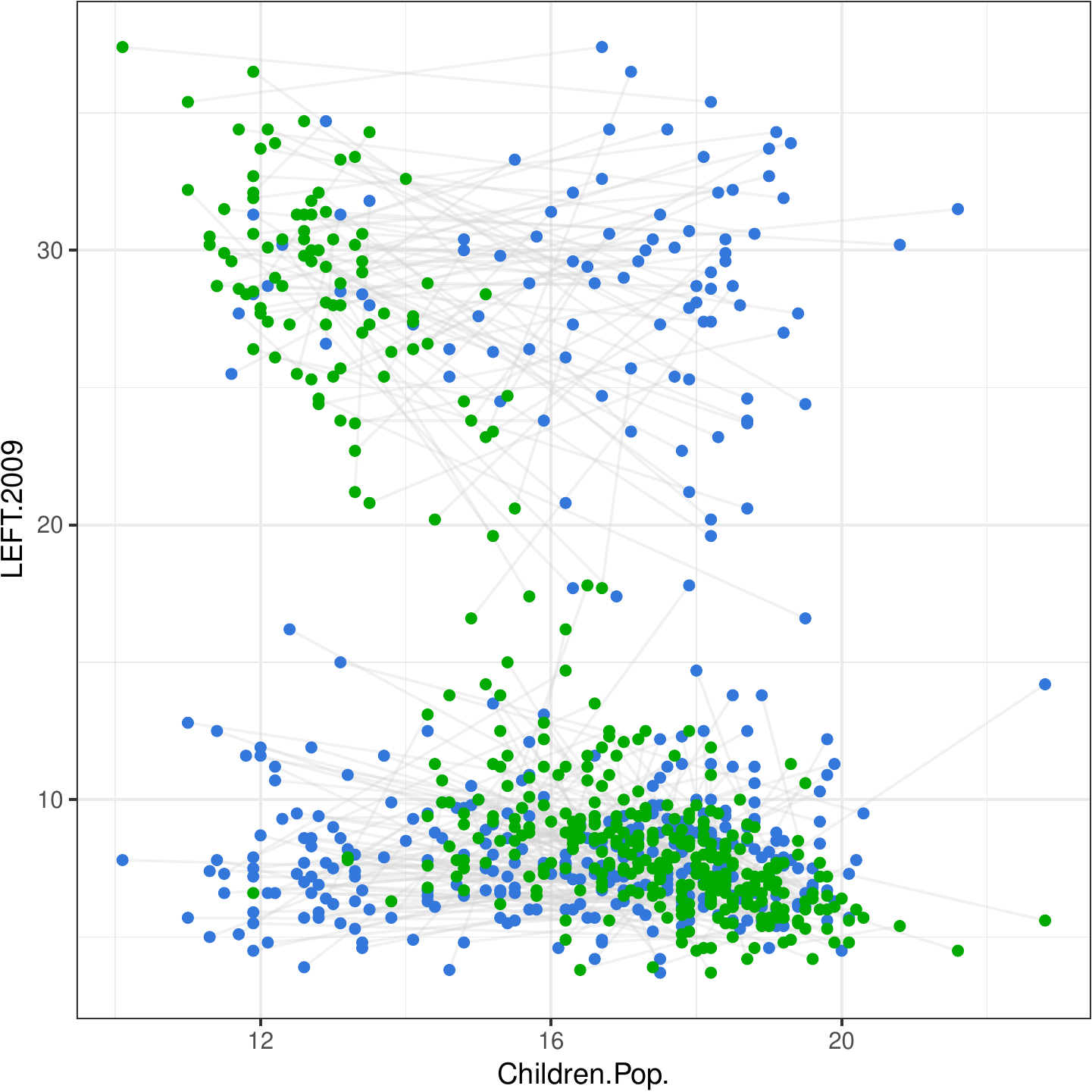} &
    \includegraphics[width=.28\textwidth]{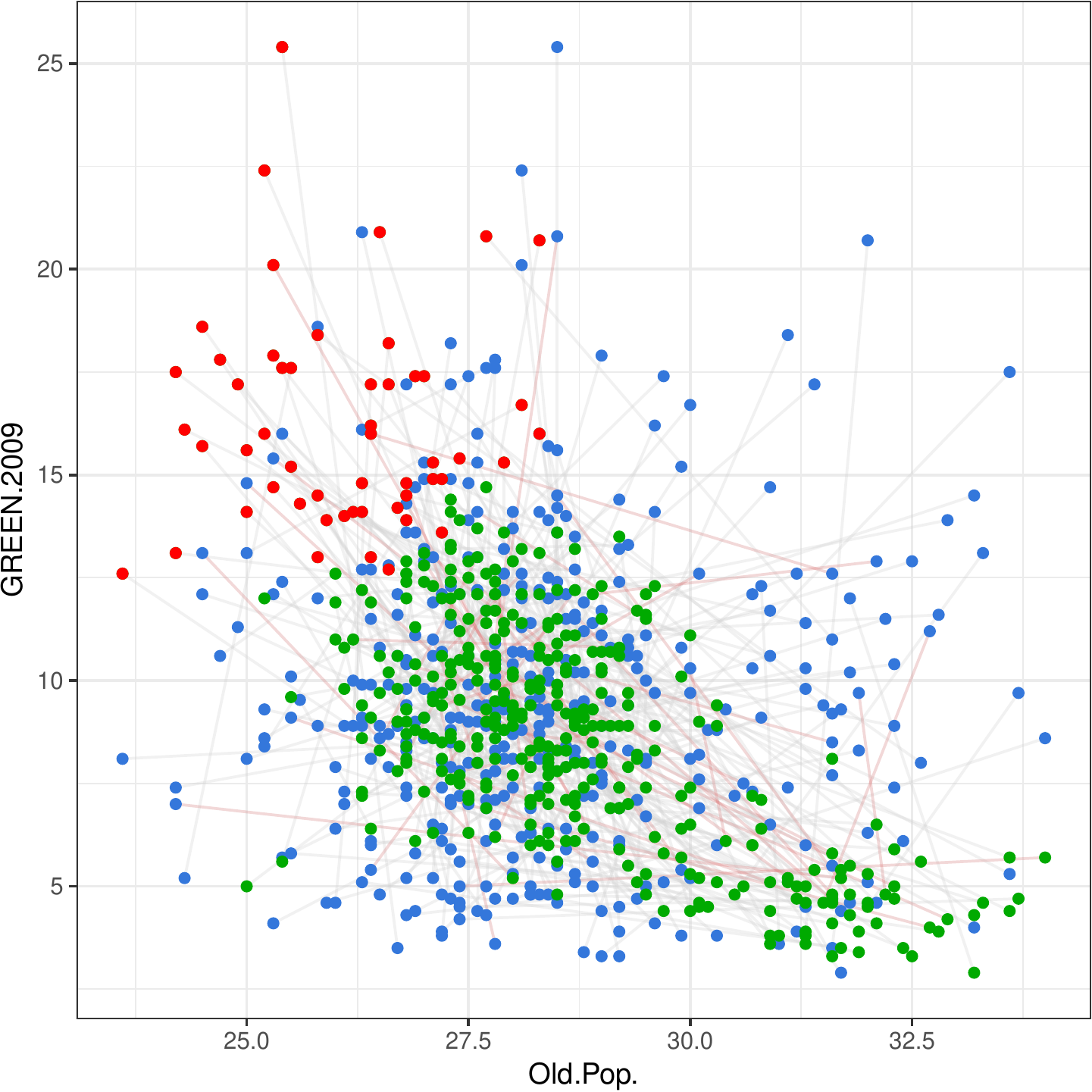} \\
            (a) & (b) & (c) 
         \end{tabular}      
  \caption{German Socio-Economic Data. (a) Task 1, the first and (b) the tenth view. (c) Task 2, the third view. \label{fig:task1}}
 \end{figure*}

\subsection{German Socio-Economic Data}
The German socio-economic dataset
\cite{boley2013,kang:2016a}\footnote{Available at
  \url{http://users.ugent.be/~bkang/software//sica/sica.zip}} consists
of records from 412 administrative districts in Germany. Each district
is represented by 46 attributes including socio-economic and political
attributes in addition to attributes such as type of the district
(rural/urban), area name/code, state, region and the geographic
coordinates of each district center. The socio-ecologic attributes
include, e.g., population density, age and education structure,
economic indicators (e.g., GDP growth, unemployment, income), and
structure of the labour market in terms of workforce in different
sectors.
The political attributes
include election results of the five major political parties (CDU/CSU,e
SPD, FDP, Green, and Left) for the German federal elections of 2005
and 2009, respectively, as well as the election participation.

\textbf{Task 1: Exploration} As we start without any prior knowledge
our first task is to explore the data using our tool to gain insight
into the main features in the data. The first view provided by the
tool is shown in Fig.~\ref{fig:task1}a.  The green points correspond
to {\sc Hypothesis 1} maintaining all correlations between attributes,
while blue points correspond to {\sc Hypothesis 2} breaking the
correlations. The locations of the actual data points, not shown in
this view, match the sample from {\sc Hypothesis 1}.  The view
reveals, not surprisingly, that there is a strong positive association
between the election results for the Left party in 2005 and 2009.  We
also see that the points in the sample from {\sc Hypothesis 1} (and
thus in the data as well) are clustered into two groups (corresponding
to different areas). We note these observations by adding two
constraints in the form of tiles that (i) connect attributes
\attr{Left 2005} and \attr{Left 2009}, and (ii) connect the points in
one of the observed clusters.  After updating the background
distributions, the samples from both hypotheses look similar, i.e.,
our constraints capture the structure of the data wrt. the attributes
in this view. Next, we request a new view.

The next view shows a strong negative association between service
 and production workforce. We add a tile constraint for all
data points and the attributes in this view and update the background
distribution. Subsequent iterations reveal further dependencies
(association between votes for the Green party (resp. SDP, CDU, FDP,
and the election participation) in 2005 and 2009, unemployment and
youth unemployment, as well as, manufacturing and production
workforce). All these views increase our understanding of the data,
although they are not very surprising.

The tenth view reveals something more interesting (Fig.~
\ref{fig:task1}b). We notice that the data points separate into two
groups, mainly on the attribute \attr{Left 2009}, though the group
with higher values of \attr{Left 2009} has lower values for
\attr{Children population}. We add a tile constraint for the set of
districts in the upper left corner, and update the background
distribution. The next view of \attr{Left 2009} and
\attr{Unemployment} shows a similar separation. We use our tool to
show the points in the tile added in the previous view and observe,
that the same sets of points are clustered also in this view.
Furthermore, the clustered sets of districts coincide almost perfectly
with the sets of clustered districts from the first view. We further
investigate this by changing the projection in the main view to
geographical coordinates, noticing that the districts in the former
East Germany together with a couple of districts in Saarland separate
in terms of more votes for the Left. Upon further analysis, it becomes
apparent that the most visible effect in the data is the division into
East and West Germany, reflected mostly in the popularity of the Left
party, but visible in, e.g., the age structure and economic
indicators, too.

\textbf{Task 2: Focusing on weaker signals} When exploring the data,
we observed several characteristics separating the former East and
West Germany. Now, we want to focus on relations between other
attributes, i.e., we want to discover relations that do not reflect
the historical division between East and West Germany.

We choose the \emph{Focus}-mode in our tool and define a focus tile
covering all districts and the following attributes: \attr{CDU 2009},
\attr{FDP 2009}, \attr{Green 2009}, \attr{SDP 2009}, \attr{Election
  participation 2009}, \attr{Population density}, all age
structure-related attributes, \attr{Income}, \attr{GDP growth 2009},
\attr{Unemployment}, and the geographical coordinates (16 attributes
in total).  Thus, we focus on the 2009 election results while leaving
the correlated 2005 elections result attributes outside the
focus. Furthermore, we leave \attr{Left 2009} outside the focus area,
due to its now known strong effect.

After defining the focus area into the tool, we obtain a view where we
observe a negative association between \attr{Unemployment} and
\attr{Children population}. We add a tile constraining this
observation and proceed. The next view shows a negative association
between \attr{Unemployment} and \attr{Election participation 2009} and
after adding a tile constraining this we proceed. The third view
(Fig.~\ref{fig:task1}c) shows \attr{Old population} and \attr{Green
  2009}. Here we see that there are roughly two kinds of outliers. In
the lower right corner districts have a larger old population and vote
less for the Green party. These correspond mainly to districts in
Eastern Germany. The other set of outliers are the districts with few
old people and voting more for the Green (the red selection in
Fig.~\ref{fig:task1}c). We add a tile constraint for both of these
outlier sets and the set of remaining districts (i.e., we partition
the districts in three sets and add a tile constraint for each set
separately). When proceeding further we are shown \attr{Middle-aged
  population} and \attr{Green 2009}.
Here we see again outliers with a high support for the Green party.
We observe that these outliers coincide almost perfectly with the
outliers with high support of the Green party from the previous
step. There is hence a set of districts characterised by a high
support of the Green party, and a large middle-aged and small old
population. Upon further inspection, we observe that these are mainly
urban districts including the largest cities in Germany.

In addition to discovering the distinction between East and West
Germany, we identified a subset of urban districts that are outliers
in terms of their support for the Green party and, e.g., the age
distribution of the population. The associations  we discovered in a straight-forward manner using visual EDA, 
are among the ones found in \cite{boley2013} with a method using pattern mining techniques. 

\subsection{UCI Image Segmentation data}
As a second use case, we consider the Image Segmentation dataset from
the UCI machine learning repository
\cite{Lichman:2013}
with 2310 samples, 19 real-valued attributes and one class attribute.

\textbf{Task 3: Focusing During Exploration} We know apriori that the
dataset contains several correlated attributes. We, however, start by
exploring the data and observe from the first view (and the
scatterplot matrix of the tool) that there is a set of 330 points
(belonging to the class \attr{sky}) in the upper-right corner clearly
separated in terms of \attr{value}, \attr{B}, \attr{G}, \attr{R}, and
\attr{Intensity} (Fig.~\ref{fig:task3}a). We add a tile constraining
all these attributes and proceed. The next view shows us that, in
fact, the same set of points also cluster in terms of \attr{value} and
\attr{excess-R}. We add another tile constraining this.

Having defined two tile constraints for the set of points from the
class \attr{sky}, we want to focus our exploration on other points,
while keeping the already defined tiles. Thus, we define a hypothesis
tiling as follows. We select the 1980 points outside the class
\attr{sky} as our focus rows $R$. For the attributes, we choose a
grouping $C=C_1\cup C_2\cup C_3$ in which the attributes related to
RGB-values are grouped into a single group $C_1=\{\attr{value},
\attr{B}, \attr{G}, \attr{R}, \attr{Intensity}, \attr{excess-B},$
$\attr{excess-G}, \attr{excess-R}\}$, and $C_2=\{\attr{saturation}\}$
and $C_3=\{\attr{hue}\}$ form individual groups.
After updating our hypotheses, we are shown the view with attributes
\attr{hue} and \attr{excess-G} (Fig.~\ref{fig:task3}b). In this view,
the set of points in the upper right (belonging to the class
\attr{grass}) form a clear cluster, and 48 points from the classes
\attr{window} and \attr{foliage} are outliers with \attr{hue} at 0 in
the middle. We add tile constraints for these sets of points, update
the background distributions and request a new view.  The next view
then (Fig.~\ref{fig:task3}c) shows further structure in terms of
\attr{saturation} and \attr{R}.

 \begin{figure*}[t]
   \centering
   \begin{tabular}{ccc}
   \includegraphics[width=.28\textwidth]{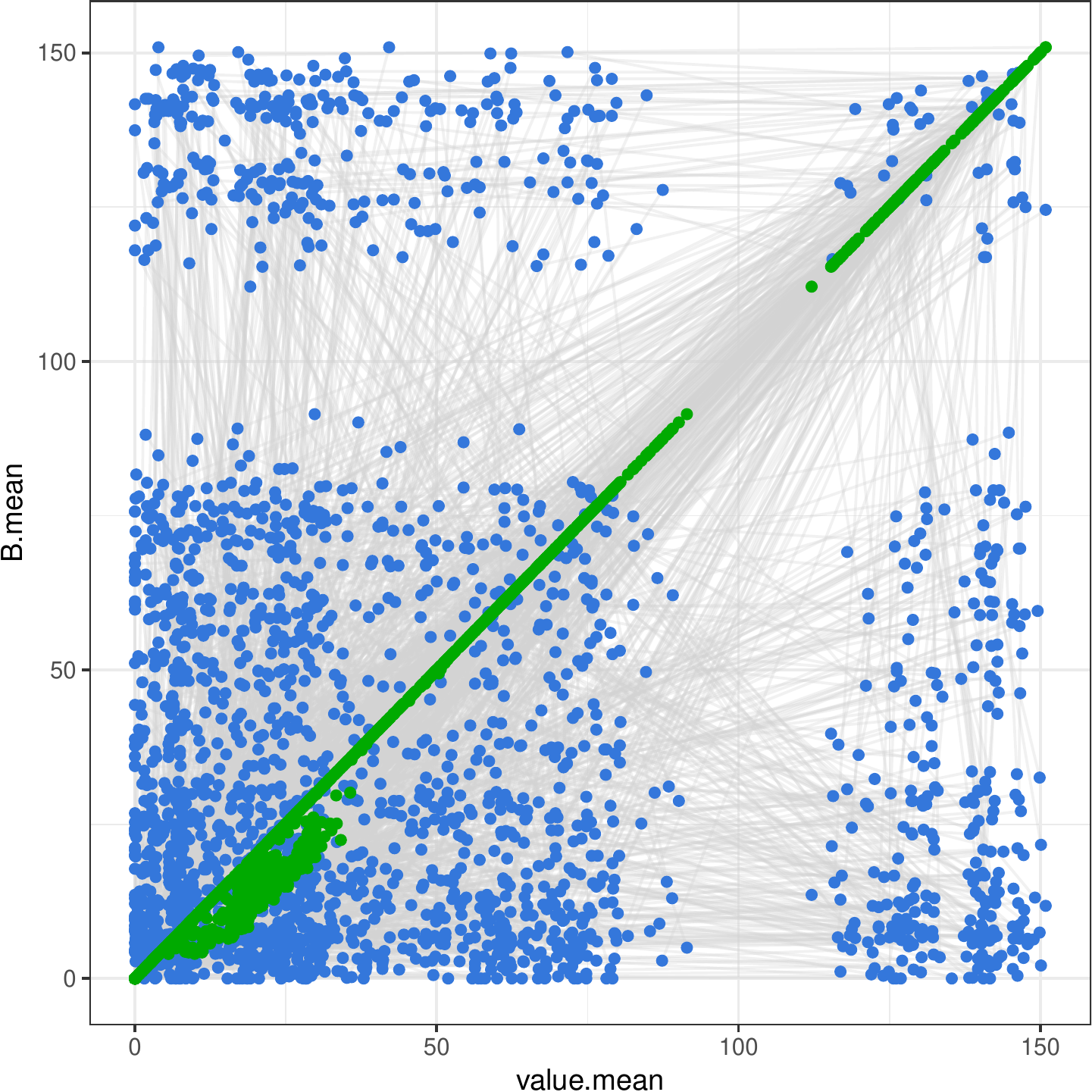} &
   \includegraphics[width=.28\textwidth]{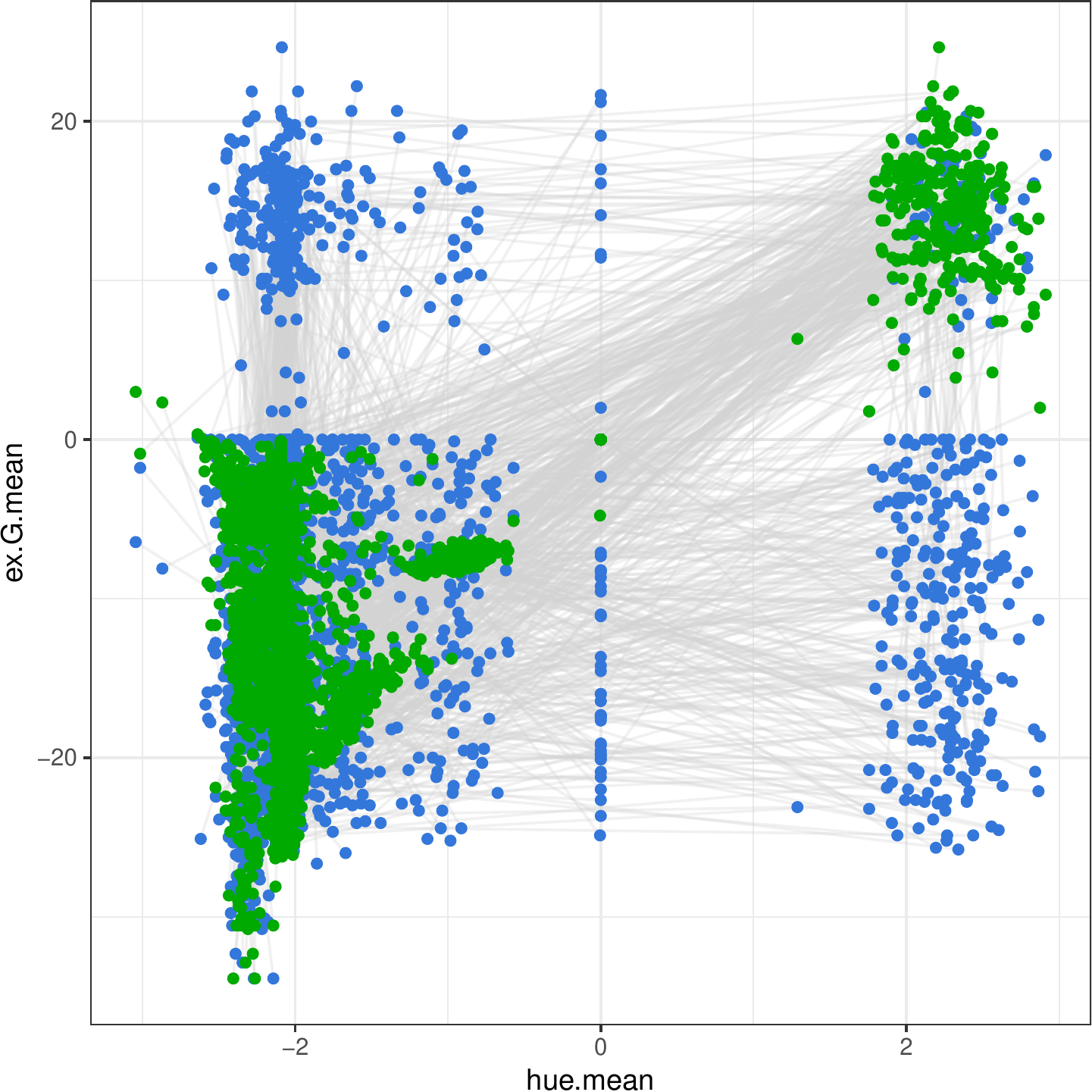} &
      \includegraphics[width=.28\textwidth]{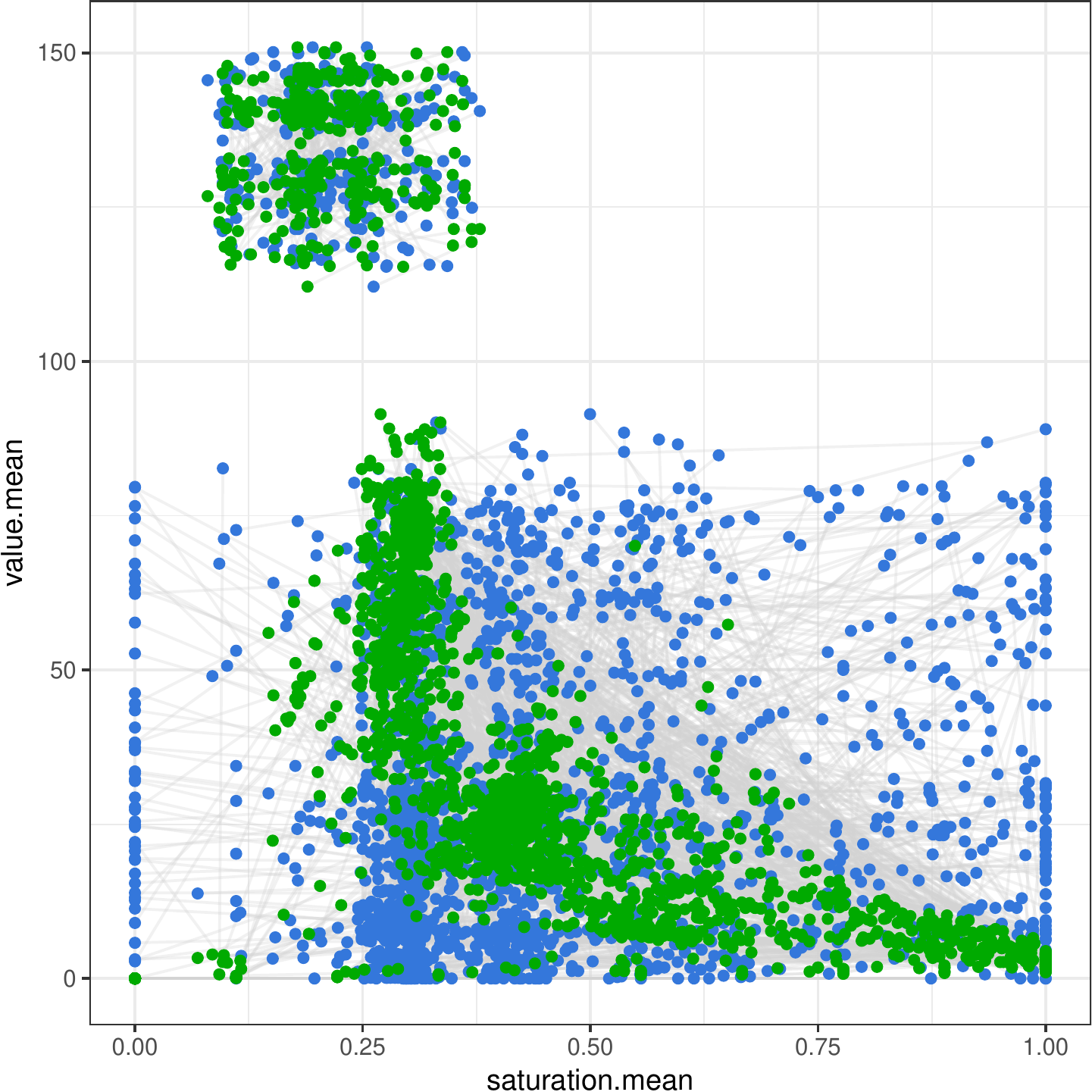} \\
      (a)  & (b) & (c) 
   \end{tabular}       
  \caption{Segmentation data. (a) Exploration-mode, the first view. (b) Comparison-mode, the first and (c) the second view. \label{fig:task3}}
 \end{figure*}

\section{Conclusions}\label{sec:conclusions}
In this paper we presented the \emph{Human-Guided Data Exploration} (HGDE)
framework,  generalising and extending previous research in the
field of iterative data mining.

The user can
direct the exploration process to relations of interest, so that
specific questions (posed as hypotheses in the framework) can be
answered. This places the iterative data mining process on an informed
basis, as the user is able to steer the process and can hence better
anticipate where the process leads. This is in contrast to prior work
where no control was possible and the data mining process was
unpredictable. It should be noted that sometimes we do want to be
surprised, though, 
and this \emph{uninformed exploration} is indeed a special
case of the  general framework presented here.

The HGDE framework is based on constrained randomisation of the data,
and the knowledge of the user is communicated by constraints on
permutations of the data. We formalise these constraints, describing
relations between the data items and the data attributes, as tiles
(subsets of rows and columns). The tiles are conceptually simple and
they are parametrised by a subset of rows and attributes of the data:
a tile means that the user has absorbed the relations of the data
points in the respective attributes. We demonstrated empirically that
our method is feasible for interactive use also with large datasets.

Although the constrained permutation scheme is conceptually simple, it
is also computationally efficient, which is an advantage over other
modelling techniques in many cases. Also, permutation of the data
using the scheme used here is a non-parametric method with very
few distributional assumptions, keeping the column margins intact (since no
new values are introduced, they are only re-ordered). This ensures
that the user is not surprised due to spurious modelling errors,
instead, 
any observed structure (or lack thereof) relates to the internal
relationships between the attributes. A further advantage of the
permutation scheme is that it is agnostic with regards to the domain
of the data attributes, i.e., since the data items are shuffled
within each column it does not matter what the domains of the columns
are.

To showcase the interactive use of the method, we developed a free
open-source tool to demonstrate these concepts, using which we also performed the
empirical evaluation in this paper. Using real-world datasets, we
demonstrated the utility of being able to focus the data exploration
process, so that novel aspects of the data are revealed, taking into
account those features of the data  the user has already
assimilated.

The HGDE framework presented in this paper, and realised in the form
of the tool, has several potential applications in explorative data
analysis tasks.  Explorative data analysis is crucial for successful
data analysis and mining. The framework presented here provides a
simple and functional approach to allowing the user to direct his or
her interest towards answering complex hypotheses concerning the
data. While there are methods for the same objectives, ours is the
first to present the task of exploration in a principled way such
that: (i) the user's background knowledge is modelled and the user is
always shown informative views, (ii) the user can ask specific
questions of the data using the same framework, and (iii) the problem
is solvable using an efficient permutation sampling methodology, as
demonstrated by the open source application published in conjunction
with this paper.

The power of human-guided data exploration stems from the fact that
typical data sets have a huge number of interesting patterns. However,
the patterns that are interesting for a user depend on the task at
hand. A non-interactive data mining method is therefore restricted to
either show generic features of data (which may be obvious to an
expert) or then output unusably many patterns (a typical problem,
e.g., in frequent pattern mining: there are easily too many patterns
for the user to absorb). Our framework is a solution to this problem:
by integrating the human background knowledge and focus---formulated
as a mathematically defined hypothesis---we can at the same time guide
the search towards topics interesting to the user at any particular
moment and take the user's prior knowledge into account in an
understandable and efficient way.

While we have demonstrated our approach on real-valued and categorical
data sets and the views used are scatterplot over pairs of attributes,
the approach is generic and can be used for different data types and
different views of the data. An interesting future problem would be to
generalise this work to other data types, such as time series,
networks etc. It would also be interesting to study how these
approaches could be used to explore the model spaces of machine
learning algorithms such as classifiers \cite{henelius2014}.

\textbf{Acknowledgements}
This work has been supported by the Academy of Finland (decisions 288814 and 313513).

\bibliographystyle{abbrv}
\bibliography{permutation_tiles}

\end{document}